\documentclass[12pt]{article}

\usepackage[margin=1in]{geometry}

\usepackage{setspace}
\usepackage{hyperref}
\usepackage{tabularx}
\usepackage{xurl}

\usepackage{graphicx}
\usepackage{amsmath, amssymb, amsthm}

\usepackage[round]{natbib}
\usepackage{lineno}
\usepackage{mathtools}

\newcommand{\R}{\mathbb{R}}
\newcommand{\E}{\mathbb{E}}
\newcommand{\grad}{\nabla}
\newcommand{\KL}{\mathrm{KL}}
\newcommand{\dd}{\mathrm{d}}
\newcommand{\Cov}{\mathrm{Cov}}
\usepackage{amssymb}
\usepackage{bm}
\usepackage{dsfont}
\usepackage[dvipsnames]{xcolor} 
\providecommand{\defeq}{\triangleq}

\DeclarePairedDelimiterX{\ip}[2]{\langle}{\rangle}{#1,\,#2}

\usepackage{amsthm}
\usepackage{thmtools}
\usepackage{booktabs}
\usepackage{tabularx}
\usepackage[most]{tcolorbox}

\declaretheorem[style=thmstyle, numberwithin=section]{theorem}
\declaretheorem[style=thmstyle, sibling=theorem]{lemma}

\declaretheorem[style=thmstyle, sibling=theorem]{proposition}
\declaretheorem[style=defstyle, sibling=theorem]{definition}
\declaretheorem[style=defstyle, sibling=theorem]{remark}

\title{Finite-Nudge Equilibrium Propagation in Thermal Ensembles}
\author{Elon Litman \\ Stanford University}
\date{}

\begin{document}

\maketitle

\begin{abstract}
We liberate Equilibrium Propagation (EP) from the limit of infinitesimal perturbations by establishing a finite-nudge foundation for local credit assignment. By modeling network states as Gibbs-Boltzmann distributions rather than deterministic points, we prove that the gradient of the difference in Helmholtz free energy between a nudged and free phase is exactly the difference in expected local energy derivatives. This validates the classic Contrastive Hebbian Learning update as an exact gradient estimator for arbitrary finite nudging, requiring neither infinitesimal approximations nor convexity. In the zero-temperature limit, we prove that the same identity reduces to the deterministic contrastive rule around any local energy basin without assuming a unique global minimum, and a subsequent small-nudge limit recovers traditional EP. Finally, we derive an equivalent representation of the same gradient as an integral of the loss--energy covariance over nudging strength, which generalizes infinitesimal EP to strong error signals that its small-nudge approximation cannot support. Numerical experiments corroborate that finite nudging provides a practical signal-to-noise advantage over infinitesimal methods during training.
\end{abstract}

\section{Introduction}
Backpropagation provides an efficient solution to the credit assignment problem in layered and recurrent networks \citep{werbos1974beyond,rumelhart1986learning} and underpins most contemporary applications of deep neural networks \citep{lecun2015deep,krizhevsky2012imagenet,he2016deep,silver2016mastering}. At the same time, the algorithm relies on a dedicated backward pass that transports errors through the network using the exact transpose of the forward weights. This weight transport requirement and the associated non-locality are widely regarded as biologically implausible and difficult to reconcile with the constraints of real neural circuits \citep{crick1989recent,grossberg1987competitive,bengio2015towards,lillicrap2020backpropagation,whittington2019theories}. Empirical and theoretical work in neuroscience instead points to synaptic plasticity rules that depend on locally available variables such as pre and postsynaptic activity \citep{hebb1949organization}, their timing \citep{gerstner2002spiking,caporale2008spike}, and possibly a small number of modulatory signals \citep{urbanczik2014learning}.

This tension has motivated a broad search for learning rules that are both powerful and local. Proposals include multivariate Hebbian and three-factor rules, feedback alignment and its variants \citep{lillicrap2016random,nokland2016direct}, target propagation \citep{lee2015difference}, and predictive coding style architectures that attempt to approximate error backpropagation with local computations \citep{whittington2017approximation,millidge2020predictive,bogacz2017tutorial}. While these approaches relax strict weight symmetry and can achieve reasonable performance, they often lack a clear global objective, or they rely on auxiliary mechanisms whose physical or biological status is unclear \citep{bengio2015towards,lillicrap2020backpropagation,whittington2019theories}.

Energy based models provide a natural framework in which to formulate local learning \citep{hopfield1982neural,ackley1985learning,lecun2006tutorial}. In this setting a network is defined by an energy function over states and parameters, and its dynamics can be viewed as a relaxation process that lowers this energy \citep{hinton2002training,du2019implicit}. Contrastive Hebbian Learning (CHL) \citep{movellan1991contrastive,oreilly1996biologically,xie2003equivalence} and Equilibrium Propagation (EP) \citep{scellier2017equilibrium,scellier2018equilibrium} are two influential schemes that exploit this structure. Both use a free phase, in which the network state is driven only by the input, and a nudged phase, in which a supervisory signal biases the system toward target configurations \citep{ernoult2019updates,laborieux2021scaling,oconnor2019representations}. Parameters are updated according to the difference between local statistics measured in the two phases, which yields a local two-phase learning rule.

Stochastic, thermal, and finite-amplitude variants of EP have been considered before. \citet{scellier2017equilibrium} themselves noted that in the presence of noise the deterministic fixed point should be replaced by a stationary distribution, and they sketched a stochastic version of the theory. More recently, \citet{massar2025equilibrium} developed quantum and thermal versions of EP, including finite-temperature response identities and a low-temperature expansion. Along a different axis, \citet{laborieux2022holomorphic} showed that in holomorphic networks, exact gradients can be recovered from finite-size oscillations of the teaching signal, sidestepping the small-nudge limit through complex-valued dynamics. The present paper differs from these works in the object it studies: for real-valued Gibbs ensembles, we define an explicit finite-nudge Helmholtz free-energy contrast $J(\theta) = A(\theta, 1) - A(\theta, 0)$ and prove that the ordinary two-phase contrastive update is exactly its gradient at every finite $T$ and $\beta$.

Despite their appeal, the theoretical status of these methods is incomplete. In its classical form, CHL is derived for architectures with symmetric weights and a single well defined energy function \citep{movellan1991contrastive,xie2003equivalence}. Under those assumptions, contrastive updates can be identified with gradients of a likelihood or related energy based objective \citep{lecun2006tutorial}. As soon as the infinitesimal limit is relaxed, which is required for noise tolerance in physical implementations and in biological circuits, the learning rule remains well defined but it is no longer obvious what global quantity it optimizes, if any \citep{xie2003equivalence,lillicrap2016random,du2019implicit}. EP addresses the weight transport issue by avoiding an explicit backward pass. It couples the energy to a supervised loss via a nudging parameter and recovers the gradient of the supervised objective in the limit of an infinitesimal perturbation \citep{scellier2017equilibrium,scellier2018equilibrium,ernoult2019updates}. In practice, however, finite nudging is required for stable learning, which introduces bias relative to the true gradient. Moreover, most analyses assume deterministic dynamics at zero temperature, in which the network state is identified with a single energy minimum \citep{ackley1985learning}, an idealization that is difficult to justify for complex nonconvex energy landscapes \citep{mezard1987spin,engel2001statistical,nishimori2001statistical}.

This paper develops a statistical mechanics foundation for contrastive learning that resolves these issues. Instead of treating the network as a deterministic energy minimizer, we model its state as a random variable distributed according to a Gibbs–Boltzmann measure at finite temperature, defined by an energy function and a task-dependent loss. Within this framework we introduce the stochastic contrastive objective, defined as the difference in Helmholtz free energy between the nudged and free phases. This objective depends only on the underlying energy based model and the loss and is well defined for arbitrary nonlinear architectures, nonconvex energy landscapes, and finite temperature dynamics \citep{mezard1987spin,engel2001statistical,nishimori2001statistical}.

We prove that the stochastic contrastive objective admits two exact and complementary gradient representations. The first expresses the gradient as the difference between the expected local energy derivatives under the nudged and free Gibbs distributions. This shows that the familiar two-phase contrastive update implements exact gradient descent on a well defined free energy objective, rather than an approximation to some other quantity, and it does so without requiring symmetric weights or an infinitesimal nudging limit. The second representation expresses the same gradient as an integral, over the nudging strength, of the covariance between the loss and the local energy derivatives. This yields a finite nudging generalization of Equilibrium Propagation in which the classical EP rule appears as a first order approximation around the free phase \citep{scellier2017equilibrium,scellier2018equilibrium,ernoult2019updates,laborieux2021scaling}. Finally, using the Gibbs variational principle, we show that the stochastic contrastive objective is a regularized proxy objective of the expected supervised loss, as it decomposes into an accuracy term and an information term that penalizes the Kullback–Leibler divergence between nudged and free distributions. This links our framework to variational inference \citep{jordan1999introduction,wainwright2008graphical,blei2017variational}, the information bottleneck \citep{tishby2000information,alemi2016deep}, and free energy based theories of brain function \citep{friston2010free,bogacz2017tutorial}.

The rest of the paper develops these results formally. We first introduce the statistical mechanics formalism and define the stochastic contrastive objective. We then derive the exact gradient formulas, establish the variational and information-theoretic interpretations, and discuss implications for local learning algorithms and their relationship to existing contrastive and equilibrium methods.
\section{The Helmholtzian Foundation for Contrastive Learning}

We begin by formalizing our theory, moving from the standard deterministic view to a more general statistical one.

\subsection{From Deterministic States to Gibbs Distributions}

Let the state of the network be a vector $s \in \mathcal{S}$, where $\mathcal{S}$ is a measurable state space (e.g., $\R^n$). Let the learnable parameters be a vector $\theta \in \Theta \subseteq \R^p$.

\begin{definition}[Energy, Loss, and Objective Kernel]
An energy function $E: \Theta \times \mathcal{S} \to \R$ is a measurable function, assumed to be continuously differentiable in $\theta$. A loss function $\ell: \mathcal{S} \to \R$ is a measurable function, independent of $\theta$. For a nudging strength $\beta \in [0, \infty)$ and a temperature $T > 0$ (often set to $1$ without loss of generality), the objective kernel is $F(\theta, \beta, s) \defeq E(\theta, s) + \beta \ell(s)$. The two-phase contrastive rule below contrasts $\beta = 0$ (free phase) with $\beta = 1$ (nudged phase); we let $\beta$ range over $[0, \infty)$ so that the same definitions cover the hard-clamping limit $\beta \to \infty$ used in classical CHL and identified with fully clamped EP by \citet{scellier2017equilibrium}.
\end{definition}

In a deterministic setting, the system's state would be a minimizer of $F$. We generalize this by defining a probability distribution over all states.

\begin{definition}[Gibbs-Boltzmann Distribution]
The Gibbs-Boltzmann distribution is a probability measure on $\mathcal{S}$ with density function $\rho_\beta(s; \theta)$ given by:
\begin{align}
\rho_\beta(s; \theta) = \frac{1}{Z_\beta(\theta)} \exp\left(-\frac{F(\theta, \beta, s)}{T}\right)
\end{align}
where $Z_\beta(\theta)$ is the partition function, a normalization constant ensuring the distribution integrates to one:
\begin{align}
Z_\beta(\theta) = \int_{\mathcal{S}} \exp\left(-\frac{F(\theta, \beta, s)}{T}\right) \dd s. 
\end{align}
We refer to $\rho_0(s; \theta)$ as the free distribution and $\rho_1(s; \theta)$ as the nudged distribution.
\end{definition}

\subsection{The Helmholtz Free Energy and The Stochastic Contrastive Objective}

The partition function is the central quantity in statistical mechanics. Its logarithm is directly related to the system's free energy.

\begin{definition}[Helmholtz Free Energy]
The Helmholtz Free Energy $A: \Theta \times [0, 1] \to \R$ is defined as:
\begin{align}
A(\theta, \beta) \defeq -T \log Z_\beta(\theta). 
\end{align}
The free energy $A(\theta, \beta)$ is the effective energy of the entire ensemble of states, accounting for both the average energy and the entropy of the distribution. It serves as the statistical generalization of the deterministic value function $\min_s F(s)$.
\end{definition}

\noindent With this, we can define our main objective function for learning.

\begin{definition}[Stochastic Contrastive Objective]
The stochastic contrastive objective $J: \Theta \to \R$ is the difference between the nudged and free Helmholtz free energies:
\begin{align}
J(\theta) \defeq A(\theta, 1) - A(\theta, 0). 
\end{align}
\end{definition}

This objective measures the thermodynamic work required to transform the system from its free state to its nudged, target-aware state. Minimizing $J(\theta)$ corresponds to adjusting parameters $\theta$ such that the target information embodied by $\ell(s)$ aligns with the natural energy landscape $E(\theta, s)$, reducing the cost of this transformation.

\subsection{Regularity Conditions}\label{as:regularity}

Our results, while general, rely on standard regularity conditions from mathematical physics to ensure that the defined quantities are well-behaved. We assume the following for all relevant $\theta$ and $\beta$:

\paragraph{Finiteness of Partition Functions.} The integrals defining the partition functions $Z_\beta(\theta)$ are finite. This requires that $\exp(-F/T)$ decays sufficiently fast over the state space $\mathcal{S}$.
\paragraph{Differentiability under the Integral Sign.} The function $E(\theta, s)$ is sufficiently regular such that differentiation with respect to $\theta$ and $\beta$ and integration over $s$ can be interchanged (i.e., the Leibniz integral rule is applicable). This is typically satisfied if $\grad_\theta E(\theta, s)$ is dominated by an integrable function of $s$. See \autoref{app:leibniz} for details.

\section{Exact Gradients for the Contrastive Objective}\label{sec:3}

We now present our main theorems, which provide exact expressions for the gradient of the objective $J(\theta)$. Free-energy derivative identities of the kind used below are standard in statistical mechanics and appear in the EP context in the stochastic extension sketched by \citet[Appendix~C]{scellier2017equilibrium} and in the finite-temperature analysis of \citet{massar2025equilibrium}. What we do here is state the identity in the notation of the finite-nudge contrastive objective $J(\theta) = A(\theta, 1) - A(\theta, 0)$, so that it can serve as the basis for the finite-$\beta$ learning rule of \autoref{thm:grad_covariance} and the controlled zero-temperature limit of \autoref{sec:zero-T}.

\begin{theorem}[Gradient as Expectation Contrast]\label{thm:grad_expectation}
Under Assumption \ref{as:regularity}, the gradient of the stochastic contrastive objective $J(\theta)$ is given exactly by the difference between the expected partial derivative of the energy under the nudged ($\beta=1$) and free ($\beta=0$) Gibbs distributions:
\begin{equation}
\grad_\theta J(\theta) = \E_{s \sim \rho_1(s; \theta)}[\grad_\theta E(\theta, s)] - \E_{s \sim \rho_0(s; \theta)}[\grad_\theta E(\theta, s)].
\end{equation}
\end{theorem}

\begin{proof}
The gradient of the Helmholtz free energy $A(\theta, \beta)$ with respect to the parameters $\theta$ is:
\begin{align}
\grad_\theta A(\theta, \beta) = \grad_\theta [-T \log Z_\beta(\theta)] = -T \frac{1}{Z_\beta(\theta)} \grad_\theta Z_\beta(\theta). 
\end{align}
By Assumption \ref{as:regularity}, we can apply the Leibniz integral rule to differentiate the partition function:
\begin{align}
\grad_\theta Z_\beta(\theta) &= \grad_\theta \int_{\mathcal{S}} \exp\left(-\frac{E(\theta, s) + \beta\ell(s)}{T}\right) \dd s \\
&= \int_{\mathcal{S}} \grad_\theta \left[ \exp\left(-\frac{F(\theta, \beta, s)}{T}\right) \right] \dd s \\
&= \int_{\mathcal{S}} \exp\left(-\frac{F(\theta, \beta, s)}{T}\right) \left(-\frac{1}{T}\right) \grad_\theta E(\theta, s) \dd s \\
&= -\frac{1}{T} \int_{\mathcal{S}} Z_\beta(\theta) \rho_\beta(s; \theta) \grad_\theta E(\theta, s) \dd s \\
&= -\frac{Z_\beta(\theta)}{T} \E_{s \sim \rho_\beta(s; \theta)}[\grad_\theta E(\theta, s)].
\end{align}
Substituting this expression back into the equation for $\grad_\theta A(\theta, \beta)$:
\begin{align}
\grad_\theta A(\theta, \beta) &= -T \frac{1}{Z_\beta(\theta)} \left( -\frac{Z_\beta(\theta)}{T} \E_{s \sim \rho_\beta(s; \theta)}[\grad_\theta E(\theta, s)] \right) \\
&= \E_{s \sim \rho_\beta(s; \theta)}[\grad_\theta E(\theta, s)].
\end{align}
This shows that the gradient of the free energy is the expectation of the energy gradient.
The gradient of our objective $J(\theta)$ follows from the linearity of the gradient operator:
\begin{align}
\grad_\theta J(\theta) = \grad_\theta[A(\theta, 1) - A(\theta, 0)] = \grad_\theta A(\theta, 1) - \grad_\theta A(\theta, 0). 
\end{align}
Applying our derived result for $\beta=1$ and $\beta=0$ immediately yields the theorem:
\begin{align}
\grad_\theta J(\theta) = \E_{s \sim \rho_1(s; \theta)}[\grad_\theta E(\theta, s)] - \E_{s \sim \rho_0(s; \theta)}[\grad_\theta E(\theta, s)]. 
\end{align}
\end{proof}

\begin{remark}[Connection to CHL and EP]
\autoref{thm:grad_expectation} provides the statistical foundation for the two-phase contrastive learning rule. A rule based on the difference of local statistics $\nabla_\theta E$ between a nudged and a free phase performs exact gradient descent on $J(\theta)$ at every finite temperature, with no infinitesimal limit and no convexity assumption on $E$. Reducing this to the classical deterministic contrastive rule as $T \to 0$ is a separate statement, which we make precise in \autoref{sec:zero-T} and prove in \autoref{app:zero-T}.
\end{remark}

Our second theorem provides an alternative expression for the gradient, connecting it to the path taken by the system as the nudging strength $\beta$ increases from 0 to 1. We first establish a lemma.

\begin{lemma}[Derivative of Free Energy w.r.t. Nudging]\label{lem:free_energy_beta}
The partial derivative of the Helmholtz free energy with respect to the nudging parameter $\beta$ is the expected value of the loss function:
\begin{align}
\frac{\partial A(\theta, \beta)}{\partial \beta} = \E_{s \sim \rho_\beta(s; \theta)}[\ell(s)]. 
\end{align}
\end{lemma}
\begin{proof}
Following a similar procedure as in \autoref{thm:grad_expectation}:
\begin{align}
\frac{\partial A(\theta, \beta)}{\partial \beta} = -T \frac{1}{Z_\beta(\theta)} \frac{\partial Z_\beta(\theta)}{\partial \beta}. 
\end{align}
Differentiating the partition function with respect to $\beta$:
\begin{align}
\frac{\partial Z_\beta(\theta)}{\partial \beta} &= \int_{\mathcal{S}} \exp\left(-\frac{E(\theta, s) + \beta\ell(s)}{T}\right) \left(-\frac{\ell(s)}{T}\right) \dd s \\
&= -\frac{1}{T} \int_{\mathcal{S}} Z_\beta(\theta) \rho_\beta(s; \theta) \ell(s) \dd s = -\frac{Z_\beta(\theta)}{T} \E_{s \sim \rho_\beta}[\ell(s)].
\end{align}
Substituting this back gives the result:
\begin{align}
\frac{\partial A(\theta, \beta)}{\partial \beta} = -T \frac{1}{Z_\beta(\theta)} \left(-\frac{Z_\beta(\theta)}{T} \E_{s \sim \rho_\beta}[\ell(s)]\right) = \E_{s \sim \rho_\beta}[\ell(s)].
\end{align}
\end{proof}

\begin{theorem}[Gradient as Integrated Covariance]\label{thm:grad_covariance} The gradient of the objective $J(\theta)$ is given by the integral of the covariance between the loss and the energy gradient, evaluated along the path of distributions from $\beta=0$ to $\beta=1$:
\begin{align}
\grad_\theta J(\theta) = -\frac{1}{T} \int_0^1 \Cov_{s \sim \rho_\beta(s; \theta)} \left[ \ell(s), \grad_\theta E(\theta, s) \right] \dd\beta.
\end{align}
\end{theorem}
\begin{proof}
By the Fundamental Theorem of Calculus and \autoref{lem:free_energy_beta}, we can write $J(\theta)$ as an integral:
\begin{align}
J(\theta) = A(\theta, 1) - A(\theta, 0) = \int_0^1 \frac{\partial A(\theta, \beta)}{\partial \beta} \dd\beta = \int_0^1 \E_{s \sim \rho_\beta(s; \theta)}[\ell(s)] \dd\beta. \end{align}
We now take the gradient of this expression with respect to $\theta$. Assumption \ref{as:regularity} allows interchanging the gradient and the integral over $\beta$:
\begin{align}
\grad_\theta J(\theta) = \int_0^1 \grad_\theta \left( \E_{s \sim \rho_\beta(s; \theta)}[\ell(s)] \right) \dd\beta.
\end{align}
We analyze the inner term $\grad_\theta \E_{s \sim \rho_\beta}[\ell(s)]$ using the log-derivative trick:
\begin{align}
\grad_\theta \E_{s \sim \rho_\beta}[\ell(s)] &= \grad_\theta \int_{\mathcal{S}} \rho_\beta(s; \theta) \ell(s) \dd s \\
&= \int_{\mathcal{S}} (\grad_\theta \rho_\beta(s; \theta)) \ell(s) \dd s \\
&= \int_{\mathcal{S}} \rho_\beta(s; \theta) (\grad_\theta \log \rho_\beta(s; \theta)) \ell(s) \dd s \\
&= \E_{s \sim \rho_\beta}[\ell(s) \grad_\theta \log \rho_\beta(s; \theta)]. 
\end{align}
The gradient of the log-density is:
\begin{align}
\grad_\theta \log \rho_\beta(s; \theta) &= \grad_\theta \left[ -\frac{F(\theta, \beta, s)}{T} - \log Z_\beta(\theta) \right] \\
&= -\frac{1}{T}\grad_\theta E(\theta, s) - \grad_\theta \log Z_\beta(\theta) \\
&= -\frac{1}{T}\grad_\theta E(\theta, s) - \frac{1}{Z_\beta(\theta)} \grad_\theta Z_\beta(\theta) \\
&= -\frac{1}{T}\grad_\theta E(\theta, s) + \frac{1}{T} \grad_\theta A(\theta, \beta).
\end{align}
From the proof of Theorem \ref{thm:grad_expectation}, we know $\grad_\theta A(\theta, \beta) = \E_{s \sim \rho_\beta}[\grad_\theta E(\theta, s)]$. Substituting this back:
\begin{align}
\grad_\theta \E_{s \sim \rho_\beta}[\ell(s)] &= \E_{s \sim \rho_\beta}\left[ \ell(s) \left( -\frac{1}{T}\grad_\theta E(\theta, s) + \frac{1}{T}\E_{s \sim \rho_\beta}[\grad_\theta E(\theta, s)] \right) \right] \\
&= -\frac{1}{T} \left( \E_{s \sim \rho_\beta}[\ell(s) \grad_\theta E(\theta, s)] - \E_{s \sim \rho_\beta}[\ell(s)] \E_{s \sim \rho_\beta}[\grad_\theta E(\theta, s)] \right) \\
&= -\frac{1}{T} \Cov_{s \sim \rho_\beta} \left[ \ell(s), \grad_\theta E(\theta, s) \right].
\end{align}
Finally, substituting this expression into our integral form for $\grad_\theta J(\theta)$ yields the theorem.
\end{proof}

\begin{remark}[Connection to Equilibrium Propagation]
\autoref{thm:grad_covariance} provides the exact, finite-$\beta$ foundation for Equilibrium Propagation. The EP update rule is derived by approximating this integral with a first-order Taylor expansion around $\beta = 0$. Closely related finite-temperature covariance--response identities have been derived by \citet{massar2025equilibrium}; \autoref{thm:grad_covariance} integrates the pointwise response identity over $\beta$ to obtain an exact expression for the finite-$\beta$ gradient of the free-energy contrast. The true gradient is therefore not a single infinitesimal Taylor coefficient but an accumulation of covariances between the loss and the local energy derivative as $\beta$ sweeps from the free to the nudged distribution. Learning adjusts $\theta$ so as to induce an anti-correlation between states with high loss and states that are sensitive to changes in $\theta$. Exact learning from finite-amplitude teaching signals has also been obtained by \citet{laborieux2022holomorphic} in the holomorphic EP framework, by moving to complex-valued dynamics and extracting gradient information from oscillatory responses; the present result stays in real-valued Gibbs ensembles and instead identifies the finite-$\beta$ free-energy objective that the two-phase update descends exactly.
\end{remark}

\subsection{The Zero-Temperature Limit}\label{sec:zero-T}

\autoref{thm:grad_expectation} is exact at every finite temperature. Its low-temperature behavior has previously been analyzed by \citet{massar2025equilibrium} through a small-$T$ expansion of the finite-temperature response identities. Here we take a complementary route: rather than expand around finite $T$, we study the $T \to 0$ limit of the finite-nudge objective $J(\theta)$ itself, and ask precisely when this limit reduces to the classical deterministic contrastive rule and what happens when it does not. The historical rules of Contrastive Hebbian Learning \citep{movellan1991contrastive,xie2003equivalence} and Equilibrium Propagation \citep{scellier2017equilibrium} are stated in the deterministic language of energy minimizers: they update parameters using $\nabla_\theta E(\theta, s^\star_1) - \nabla_\theta E(\theta, s^\star_0)$, where $s^\star_\beta$ is a state that minimizes $F(\theta, \beta, \cdot)$. To keep the finite-temperature theory tethered to its classical antecedents, we check that our result actually reduces to this deterministic rule as $T \to 0$.

As $T$ decreases, the Gibbs weight $\exp(-F/T)$ becomes increasingly peaked around the smallest values of $F$; in the limit, the distribution collapses onto the minimizer set
\begin{align}
M_\beta(\theta) = \operatorname*{argmin}_{s \in \mathcal{S}} F(\theta, \beta, s).
\end{align}
The expectation of any continuous function of $s$ then reduces to an average over $M_\beta(\theta)$. Applied to $\nabla_\theta E$, this turns \autoref{thm:grad_expectation} into a difference of derivatives evaluated at minimizers, which is the classical contrastive rule. We work through the arguments in \autoref{app:zero-T} and record the resulting statements here.

For the rest of this subsection we assume that $\mathcal{S}$ is compact (or that $F$ is coercive on non-compact $\mathcal{S}$), that $E$ and $\ell$ are continuous in $s$, and that $\nabla_\theta E$ is jointly continuous in $(\theta, s)$. We do \emph{not} assume convexity of $E$, uniqueness of the minimizer, or any smoothness in $s$. Write $m_\beta(\theta) = \min_s F(\theta, \beta, s)$ for the minimum value.

\paragraph{Global recovery.} As $T \to 0$, the free energy converges, $A(\theta, \beta) \to m_\beta(\theta)$, and the Gibbs distribution $\rho_\beta$ concentrates on $M_\beta(\theta)$ (\autoref{prop:free-energy-limit}). Every subsequential limit of $\nabla_\theta A(\theta, \beta) = \E_{\rho_\beta}[\nabla_\theta E]$ lies in the closed convex hull of $\{\nabla_\theta E(\theta, s) : s \in M_\beta(\theta)\}$. When $m_\beta$ is differentiable at $\theta$, which happens automatically for a unique minimizer, but also whenever several minimizers all produce the same $\nabla_\theta E$, the convex hull collapses to a single point (\autoref{prop:grad-limit}), and
\begin{align}\label{eq:global-zero-T}
\lim_{T \to 0} \nabla_\theta A(\theta, \beta) = \nabla_\theta m_\beta(\theta), \qquad \lim_{T \to 0} \nabla_\theta J(\theta) = \nabla_\theta[m_1(\theta) - m_0(\theta)].
\end{align}
This is the classical two-phase contrastive update. Convexity is one route to differentiability of $m_\beta$; it is not the only one, and the result above applies at every $\theta$ where $m_\beta$ is differentiable, whether or not $E$ is convex in $s$.

\paragraph{Local recovery.} Physical and biological networks relax into a nearby energy basin rather than the global minimum, and CHL and EP are usually analyzed in that setting. Applying the same argument inside a bounded neighborhood $U$ in which $F(\theta, \beta, \cdot)$ has a unique local minimizer $s^\star_\beta(\theta)$ separated from $\partial U$ by a positive gap gives
\begin{align}
\lim_{T \to 0} \nabla_\theta A^U(\theta, \beta) = \nabla_\theta E(\theta, s^\star_\beta(\theta)),
\end{align}
where $A^U$ is the free energy defined by integrating only over $U$ (\autoref{prop:local-limit}). Subtracting the $\beta = 1$ and $\beta = 0$ cases reproduces the deterministic two-phase update around whichever local equilibrium the sampler occupies. Nothing in this statement requires that the sampled basin be the global one.

\paragraph{Classical infinitesimal EP.} A further small-$\beta$ limit along a smooth branch $s_\beta(\theta)$ of local equilibria recovers the original EP identity of \citet{scellier2017equilibrium}. Because $s_\beta$ satisfies the stationarity condition $\nabla_s F(\theta, \beta, s_\beta) = 0$, the state-derivative contributions to $d/d\beta$ and $d/d\theta$ of $F(\theta, \beta, s_\beta(\theta))$ vanish, and a direct calculation (\autoref{prop:small-beta}) gives
\begin{align}\label{eq:infinitesimal-EP}
\lim_{\beta \to 0} \frac{\nabla_\theta E(\theta, s_\beta) - \nabla_\theta E(\theta, s_0)}{\beta} = \nabla_\theta \ell(s_0(\theta)).
\end{align}
The finite-$T$ finite-$\beta$ identity of \autoref{thm:grad_expectation}, the deterministic finite-$\beta$ rule from the local recovery above, and this small-$\beta$ identity together form a single family of results, connected by first taking $T \to 0$ around a local basin and then taking $\beta \to 0$.

\paragraph{Degenerate minima.} What happens if $M_\beta(\theta)$ contains several isolated points $s_1, \dots, s_k$ that produce genuinely different values of $\nabla_\theta E$? The finite-$T$ theory still has a well-defined zero-temperature limit, but the limit is not any one of the $\nabla_\theta E(\theta, s_i)$ individually. It is a weighted average of them, and the weights are set by local curvature. A Laplace expansion around each minimizer (\autoref{prop:degenerate}) gives
\begin{align}\label{eq:curvature-weight}
\lim_{T \to 0} \E_{\rho_\beta}[\nabla_\theta E] = \sum_{i=1}^k w_i \nabla_\theta E(\theta, s_i), \qquad w_i = \frac{\det(H_i)^{-1/2}}{\sum_{j=1}^k \det(H_j)^{-1/2}},
\end{align}
where $H_i$ is the Hessian of $F(\theta, \beta, \cdot)$ at $s_i$. Flatter minima carry larger weight. We verify this curvature-weighted sampling numerically in \autoref{app:degenerate-numerical}, demonstrating that an exact 1D Gibbs ensemble correctly converges to the Laplace limits rather than uniform sampling. Uniform weighting is recovered only in the special case where all Hessian determinants agree, and is not the generic zero-temperature behavior for a continuous state space. Indeed, the finite-temperature theory continues to give a well-defined gradient, and that gradient is a specific curvature-weighted mixture over the active minimizers.

\section{Connection to Supervised Learning}

We now investigate how minimizing our objective $J(\theta)$ relates to minimizing the standard supervised loss, which we can define as the expected loss under the free distribution, $\mathcal{L}_{\text{sup}}(\theta) \defeq \E_{s \sim \rho_0}[\ell(s)]$. We show that $J(\theta)$ is a regularized proxy objective of this quantity. We first state the Gibbs variational principle.

\begin{lemma}[Gibbs Variational Principle]\label{lem:gibbs}
The Helmholtz free energy $A(\theta, \beta)$ is the minimum of the variational free energy functional over all probability distributions $q(s)$:
\begin{align}
A(\theta, \beta) = \min_{q} \bigg\{ \E_{s \sim q(s)}[E(\theta, s) + \beta\ell(s)] - T S(q) \bigg\}, 
\end{align}
where 
\begin{align}
S(q) = -\int q(s) \log q(s) \dd s 
\end{align}
is the differential entropy of the distribution $q$. The minimum is achieved uniquely at $q(s) = \rho_\beta(s; \theta)$.
\end{lemma}

\begin{theorem}[Variational Bound on Supervised Loss]\label{thm:variational_bound}
The stochastic contrastive objective $J(\theta)$ provides a tight variational lower bound on the expected supervised loss under the free distribution:
\begin{align}
J(\theta) \le \E_{s \sim \rho_0(s; \theta)}[\ell(s)].
\end{align}
\end{theorem}
\begin{proof}
From the Gibbs Variational Principle (\autoref{lem:gibbs}), the free energy $A(\theta, 1)$ is the minimum of the variational free energy functional. Therefore, for any trial distribution $q(s)$, we have:
\begin{align}
A(\theta, 1) \le \E_{s \sim q(s)}[E(\theta, s) + \ell(s)] - T S(q).
\end{align}
Let us choose the free distribution $q(s) = \rho_0(s; \theta)$ as our specific trial distribution. Substituting this in:
\begin{align}
A(\theta, 1) &\le \E_{s \sim \rho_0}[E(\theta, s) + \ell(s)] - T S(\rho_0) \\
&= \E_{s \sim \rho_0}[E(\theta, s)] + \E_{s \sim \rho_0}[\ell(s)] - T S(\rho_0).
\end{align}
By definition, the Helmholtz free energy of the free system is 
\begin{align}
A(\theta, 0) = \E_{s \sim \rho_0}[E(\theta, s)] - T S(\rho_0).
\end{align}
Substituting this into our inequality gives:
\begin{align}
A(\theta, 1) \le A(\theta, 0) + \E_{s \sim \rho_0}[\ell(s)]. 
\end{align}
Rearranging the terms yields the final result:
\begin{align}
J(\theta) = A(\theta, 1) - A(\theta, 0) \le \E_{s \sim \rho_0}[\ell(s)]. 
\end{align}
\end{proof}

\begin{remark}
\autoref{thm:variational_bound} establishes $J(\theta)$ as a lower bound, but the justification for minimizing it is twofold. First, for non-negative losses, $J(\theta)=0$ if and only if $\E_{\rho_0}[\ell(s)] = 0$; thus, global minima of the objective are global minima of the supervised loss. Second, as we will see in \autoref{thm:decomposition}, $J(\theta)$ effectively minimizes the loss under the nudged distribution while simultaneously pulling the free distribution toward it via KL regularization.
\end{remark}

\section{Information-Theoretic Interpretation}

We now reveal that $J(\theta)$ can be decomposed into two competing terms: an information-theoretic cost that measures the distance between the free and nudged distributions, and a performance cost that measures the residual loss in the nudged state. This decomposition provides a direct link to the principles of variational inference and the information bottleneck.

\begin{theorem}[Information-Performance Decomposition]\label{thm:decomposition}
Under the same conditions as before, $J(\theta)$ can be exactly decomposed as:
\begin{equation}
J(\theta) = \E_{s \sim \rho_1(s; \theta)}[\ell(s)] + T \KL(\rho_1(s; \theta) \,\|\, \rho_0(s; \theta))
\end{equation}
where $\KL(p\|q)$ is the Kullback-Leibler (KL) divergence between distributions $p$ and $q$.
\end{theorem}
\begin{proof}
The proof proceeds directly from the definition of the KL divergence. Let $\rho_1 \triangleq \rho_1(s; \theta)$ and $\rho_0 \triangleq \rho_0(s; \theta)$ for brevity. The KL divergence from the free distribution $\rho_0$ to the nudged distribution $\rho_1$ is defined as:
\begin{align}
\KL(\rho_1 \,\|\, \rho_0) = \int_{\mathcal{S}} \rho_1(s) \log\left(\frac{\rho_1(s)}{\rho_0(s)}\right) \dd s. 
\end{align}
We first analyze the logarithm term. By the definition of the Gibbs distributions:
\begin{align}
\frac{\rho_1(s)}{\rho_0(s)} &= \frac{Z_0(\theta)}{Z_1(\theta)} \frac{\exp\left(-(E(\theta, s) + \ell(s))/T\right)}{\exp\left(-E(\theta, s)/T\right)} \\
&= \frac{Z_0(\theta)}{Z_1(\theta)} \exp\left(-\frac{\ell(s)}{T}\right).
\end{align}
Taking the logarithm of this ratio:
\begin{align}
\log\left(\frac{\rho_1(s)}{\rho_0(s)}\right) &= \log\left(\frac{Z_0(\theta)}{Z_1(\theta)}\right) - \frac{\ell(s)}{T} \\
&= -(\log Z_1(\theta) - \log Z_0(\theta)) - \frac{\ell(s)}{T}.
\end{align}
Recalling the definition of the Helmholtz Free Energy, $A(\theta, \beta) = -T \log Z_\beta(\theta)$, we have $\log Z_\beta(\theta) = -A(\theta, \beta)/T$. Substituting this in:
\begin{align}
\log\left(\frac{\rho_1(s)}{\rho_0(s)}\right) &= -\left(-\frac{A(\theta, 1)}{T} - \left(-\frac{A(\theta, 0)}{T}\right)\right) - \frac{\ell(s)}{T} \\
&= \frac{A(\theta, 1) - A(\theta, 0)}{T} - \frac{\ell(s)}{T}.
\end{align}
By the definition of our objective, $J(\theta) = A(\theta, 1) - A(\theta, 0)$, this simplifies to:
\begin{align}
\log\left(\frac{\rho_1(s)}{\rho_0(s)}\right) = \frac{J(\theta) - \ell(s)}{T}. 
\end{align}
Now, we substitute this back into the definition of the KL divergence:
\begin{align}
\KL(\rho_1 \,\|\, \rho_0) &= \int_{\mathcal{S}} \rho_1(s) \left(\frac{J(\theta) - \ell(s)}{T}\right) \dd s \\
&= \frac{1}{T} \int_{\mathcal{S}} \rho_1(s) J(\theta) \dd s - \frac{1}{T} \int_{\mathcal{S}} \rho_1(s) \ell(s) \dd s.
\end{align}
Since $J(\theta)$ is a constant with respect to the integration variable $s$, and $\int \rho_1(s) \dd s = 1$, the first term simplifies. The second term is, by definition, the expectation of $\ell(s)$ under $\rho_1$.
\begin{align}
\KL(\rho_1 \,\|\, \rho_0) = \frac{J(\theta)}{T} - \frac{1}{T} \E_{s \sim \rho_1}[\ell(s)]. 
\end{align}
Rearranging this equation to solve for $J(\theta)$ yields the theorem:
\begin{align}
J(\theta) = \E_{s \sim \rho_1}[\ell(s)] + T \KL(\rho_1 \,\|\, \rho_0). 
\end{align}
\end{proof}

\begin{remark}[Mechanism of Learning]
While \autoref{thm:variational_bound} establishes $J(\theta)$ as a lower bound on the supervised loss $\E_{\rho_0}[\ell]$, minimizing a lower bound does not strictly guarantee minimization of the target. \autoref{thm:decomposition} resolves this by identifying $J(\theta)$ as a regularized surrogate rather than a loose bound. The objective explicitly minimizes the loss in the nudged phase ($\rho_1$) while simultaneously minimizing the KL divergence between $\rho_1$ and $\rho_0$. This compels the free phase to emulate the low-energy statistics of the nudged phase, effectively distilling the supervisory signal into the network's natural dynamics.
\end{remark}

\section{Results}\label{sec:results}

We evaluated finite–nudge Equilibrium Propagation (EP) on Fashion–MNIST using a single–layer energy–based network with $\texttt{tanh}$ units. Holding architecture and optimization constant, we varied only the learning rule and nudging strength $\beta$ (full methodological details, network architecture, and hyperparameters are provided in \autoref{app:fashion-mnist}). \autoref{fig:thermodynamic-validation} summarizes the structural, statistical, and practical implications of our findings. First, we assessed classification performance (Fig.~\ref{fig:thermodynamic-validation}, Panel~C). We compared four training schemes: classical infinitesimal EP ($\beta = 0.01$), finite–nudge EP ($\beta = 1.0$), a discrete path–integral variant, and standard backpropagation. Infinitesimal EP failed, stalling near chance (20–30\% accuracy). Conversely, both finite–nudge and path–integral EP rapidly achieved $\sim 80\%$ accuracy, closely tracking the backpropagation baseline. These results confirm that large-$\beta$ contrastive learning is a competitive training mechanism, while the infinitesimal regime is ineffective for this task.

Next, we analyzed the signal–to–noise ratio (SNR) of the state updates (Fig.~\ref{fig:thermodynamic-validation}, Panel~B). We measured the SNR of the activity difference $\Delta s = s_\beta - s_0$ over repeated Langevin runs. For $\beta \lesssim 10^{-2}$, the update signal was indistinguishable from sampling noise. As $\beta$ approached $1$, SNR improved by an order of magnitude. This empirically confirms that finite nudging unlocks a signal regime inaccessible to the classical small-$\beta$ limit.

Finally, we inspected the orientation of the parameter updates (Fig.~\ref{fig:thermodynamic-validation}, Panel~A). We measured the cosine similarity between the practical contrastive update and two benchmarks: the exact supervised gradient ($\nabla \mathcal{L}_{\mathrm{sup}}$) and the Monte Carlo estimate of the free–energy gradient ($\nabla J_\beta$). Alignment was negligible near $\beta \approx 10^{-3}$ but increased monotonically, reaching $\sim 0.5$ at $\beta = 1$. Thus, finite–nudge EP aligns with the true thermodynamic objective and, by operating at large $\beta$, overcomes the noise limitations that cripple infinitesimal approaches. To verify these results without the confounder of Langevin sampling noise, \autoref{app:exact-ebm} provides an additional experiment using exact mathematical marginalization on a smaller dataset.

\begin{figure}[htbp]
  \centering
  \includegraphics[width=\textwidth]{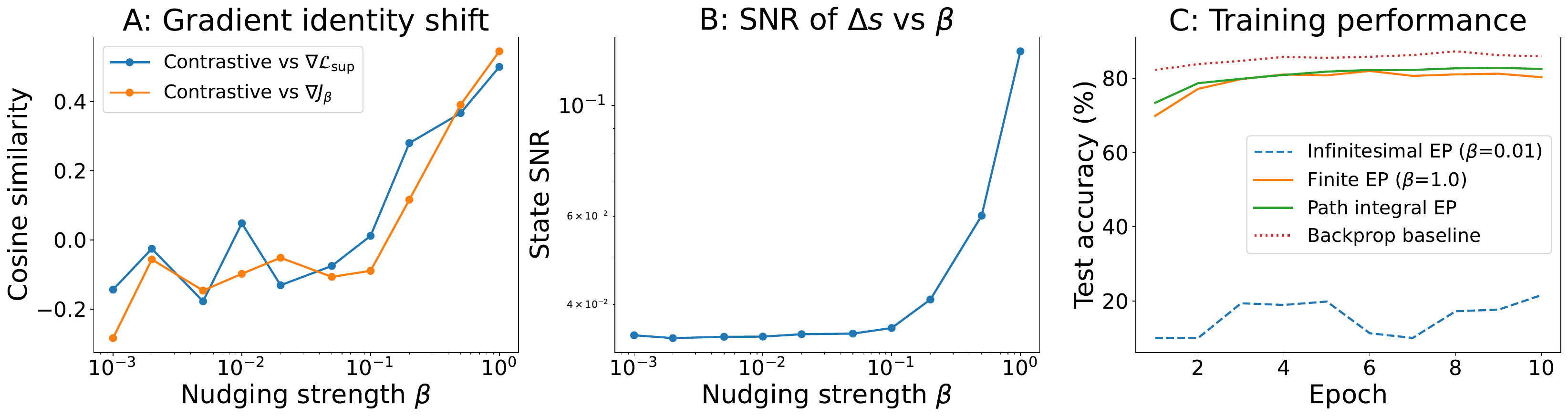}
  \caption{\textbf{Thermodynamic validation on Fashion--MNIST.}
Experiments utilize a single hidden-layer energy-based network with $\texttt{tanh}$ units.
\textbf{(A) Gradient Alignment:} Cosine similarity between the practical contrastive update $\hat{g}(\beta)$ and two references: the supervised backprop gradient $\nabla \mathcal{L}_{\mathrm{sup}}$ and the true free-energy gradient $\nabla J_\beta$. Alignment improves monotonically with $\beta$, confirming that large nudges remain gradient-like.
\textbf{(B) Signal-to-Noise Ratio:} SNR of the state perturbation $\Delta s = s_\beta - s_0$. Finite nudging ($\beta \to 1$) yields high SNR, whereas infinitesimal nudges ($\beta \lesssim 10^{-2}$) are dominated by sampling noise.
\textbf{(C) Test Accuracy:} Finite-nudge ($\beta=1.0$) and path-integral EP achieve $\sim 80\%$ accuracy, closely tracking standard backprop. Classical infinitesimal EP ($\beta=0.01$) fails to learn.
  }
  \label{fig:thermodynamic-validation}
\end{figure}

\section{Conclusion}

This work decouples Equilibrium Propagation from the infinitesimal limit. We proved that finite-nudge learning is not a biased approximation of backpropagation, but exact gradient descent on the Helmholtz free energy difference. This statistical view reinterprets the \emph{error} of large nudges as a variational term that minimizes the divergence between free and target distributions. Empirically, unlocking the finite-nudge regime solves the signal-to-noise problem that cripples infinitesimal approaches.

\bibliography{references}

\newpage
\appendix

\section{Leibniz Integral Rule}\label{app:leibniz}

The proofs in Section \ref{sec:3} rely on the ability to interchange the order of integration and differentiation. This is justified by the Leibniz integral rule (or differentiation under the integral sign). A general version states:

\begin{theorem}[Leibniz Integral Rule]
Let $\Omega$ be an open set in $\R^p$, and let $f(s, \theta)$ be a function defined on $\mathcal{S} \times \Omega$. Assume that for all $\theta \in \Omega$, $f(s, \theta)$ is an integrable function of $s$. Assume that for almost every $s \in \mathcal{S}$, the partial derivative $\frac{\partial f}{\partial \theta_i}$ exists for all $\theta \in \Omega$. If there exists an integrable function $g(s)$ such that for all $\theta \in \Omega$, $|\frac{\partial f}{\partial \theta_i}(s, \theta)| \le g(s)$ for almost every $s \in \mathcal{S}$, then for all $\theta \in \Omega$:
\begin{align}
\frac{\partial}{\partial \theta_i} \int_{\mathcal{S}} f(s, \theta) \dd s = \int_{\mathcal{S}} \frac{\partial f}{\partial \theta_i}(s, \theta) \dd s.
\end{align}
\end{theorem}
In our context, $f(s, \theta)$ is the Gibbs factor $\exp(-F/T)$, and Assumption \ref{as:regularity} requires that its derivative with respect to $\theta_i$ is dominated by a function $g(s)$ that is integrable over $\mathcal{S}$. This is satisfied by a wide range of well-behaved energy functions used in practice.

\section{The Zero-Temperature Limit: Precise Statements and Proofs}\label{app:zero-T}

This appendix supports the discussion in \autoref{sec:zero-T}. Throughout, we assume $\mathcal{S}$ is compact, that $E$ and $\ell$ are continuous in $s$, that $E$ is continuously differentiable in $\theta$, and that $\nabla_\theta E(\theta, s)$ is jointly continuous in $(\theta, s)$. These conditions extend the finite-temperature regularity of \autoref{as:regularity}, which only assumed $\theta$-differentiability under the integral sign; the additional structure is needed to control the behavior of $F(\theta, \beta, \cdot)$ near its minimizers. Compactness of $\mathcal{S}$ can be relaxed to coercivity of $F$ on non-compact $\mathcal{S}$ by a standard truncation argument, which we do not repeat here. For each $(\theta, \beta)$ we write
\begin{align}
m_\beta(\theta) = \min_{s \in \mathcal{S}} F(\theta, \beta, s), \qquad M_\beta(\theta) = \operatorname*{argmin}_{s \in \mathcal{S}} F(\theta, \beta, s).
\end{align}
By continuity of $F$ and compactness of $\mathcal{S}$, both are well defined and $M_\beta(\theta)$ is non-empty.

\subsection{Convergence of Free Energy and Concentration}

\begin{proposition}\label{prop:free-energy-limit}
For each $(\theta, \beta)$, $A(\theta, \beta) \to m_\beta(\theta)$ as $T \to 0$. Moreover, the Gibbs distribution concentrates on $M_\beta(\theta)$: for every open $U \supset M_\beta(\theta)$, $\rho_\beta(U; \theta) \to 1$.
\end{proposition}

\begin{proof}
Abbreviate $m = m_\beta(\theta)$ and $F(s) = F(\theta, \beta, s)$. Since $F \ge m$ throughout $\mathcal{S}$,
\begin{align}
Z_\beta(\theta) = \int_\mathcal{S} \exp(-F/T) \dd s \le |\mathcal{S}| \exp(-m/T),
\end{align}
so $A(\theta, \beta) = -T \log Z_\beta(\theta) \ge m - T \log |\mathcal{S}|$. For the matching upper bound, fix $\varepsilon > 0$ and set $U_\varepsilon = \{s \in \mathcal{S} : F(s) < m + \varepsilon\}$, an open neighborhood of $M_\beta$ with positive Lebesgue measure. Then
\begin{align}
Z_\beta(\theta) \ge \int_{U_\varepsilon} \exp(-F/T) \dd s \ge |U_\varepsilon| \exp(-(m + \varepsilon)/T),
\end{align}
so $A(\theta, \beta) \le m + \varepsilon - T \log |U_\varepsilon|$. Combining and letting $T \to 0$ then $\varepsilon \to 0$ gives $A(\theta, \beta) \to m$. For concentration, let $U$ be any open set containing $M_\beta$. The complement $\mathcal{S} \setminus U$ is compact and contains no minimizer, so
\begin{align}
\delta := \inf_{s \in \mathcal{S} \setminus U} F(s) - m > 0.
\end{align}
For any $\varepsilon \in (0, \delta)$,
\begin{align}
\rho_\beta(\mathcal{S} \setminus U) = \frac{\int_{\mathcal{S} \setminus U} \exp(-F/T) \dd s}{Z_\beta(\theta)} \le \frac{|\mathcal{S}| \exp(-(m + \delta)/T)}{|U_\varepsilon| \exp(-(m + \varepsilon)/T)} = \frac{|\mathcal{S}|}{|U_\varepsilon|} \exp\!\bigl(-(\delta - \varepsilon)/T\bigr),
\end{align}
which tends to $0$ as $T \to 0$.
\end{proof}

\subsection{Convergence of the Gradient}

\begin{proposition}\label{prop:grad-limit}
Every subsequential limit of $\nabla_\theta A(\theta, \beta) = \E_{\rho_\beta}[\nabla_\theta E(\theta, s)]$ as $T \to 0$ lies in the closed convex hull of $\{\nabla_\theta E(\theta, s) : s \in M_\beta(\theta)\}$. If $m_\beta$ is differentiable at $\theta$, then $\nabla_\theta E(\theta, s) = \nabla_\theta m_\beta(\theta)$ for every $s \in M_\beta(\theta)$, and consequently
\begin{align}
\lim_{T \to 0} \nabla_\theta A(\theta, \beta) = \nabla_\theta m_\beta(\theta).
\end{align}
\end{proposition}

\begin{proof}
\autoref{thm:grad_expectation} gives $\nabla_\theta A(\theta, \beta) = \E_{\rho_\beta}[\nabla_\theta E(\theta, s)]$. Because $\nabla_\theta E$ is continuous on the compact set $\{\theta\} \times \mathcal{S}$, it is bounded there. Any subsequence of Gibbs distributions on $\mathcal{S}$ has a weakly convergent sub-subsequence, and any such weak limit must be supported on $M_\beta(\theta)$ by \autoref{prop:free-energy-limit}. Continuity of $\nabla_\theta E$ then gives $\E_{\rho_\beta}[\nabla_\theta E] \to \int_{M_\beta} \nabla_\theta E \dd\rho^\star$ along this sub-subsequence, which lies in the closed convex hull of $\nabla_\theta E(\theta, M_\beta(\theta))$. Now suppose $m_\beta$ is differentiable at $\theta$ with gradient $g$, and pick any $s^\star \in M_\beta(\theta)$. For every direction $v$ and every small $t > 0$,
\begin{align}
m_\beta(\theta + tv) \le F(\theta + tv, \beta, s^\star) &= F(\theta, \beta, s^\star) + t \langle \nabla_\theta E(\theta, s^\star), v \rangle + o(t) \\ 
&= m_\beta(\theta) + t \langle \nabla_\theta E(\theta, s^\star), v \rangle + o(t),
\end{align}
using $s^\star \in M_\beta(\theta)$ in the last step. Dividing by $t$ and letting $t \to 0$ gives $\langle g, v \rangle \le \langle \nabla_\theta E(\theta, s^\star), v \rangle$. Applying the same argument to $-v$ gives the reverse inequality. Hence $\nabla_\theta E(\theta, s^\star) = g$ for every $s^\star \in M_\beta(\theta)$, the convex hull is the single point $g$, and the expectations converge to $g$.
\end{proof}

Applied at $\beta = 1$ and $\beta = 0$, \autoref{prop:grad-limit} says that whenever $m_1$ and $m_0$ are both differentiable at $\theta$,
\begin{align}
\lim_{T \to 0} \nabla_\theta J(\theta) = \nabla_\theta[m_1(\theta) - m_0(\theta)],
\end{align}
which is the classical deterministic contrastive gradient. Neither convexity of $E$ nor uniqueness of the minimizer is needed.

\subsection{Restriction to a Local Basin}

For any bounded open $U \subseteq \mathcal{S}$ define
\begin{align}
Z^U_\beta(\theta) = \int_U \exp\!\bigl(-F(\theta, \beta, s)/T\bigr) \dd s, \qquad A^U(\theta, \beta) = -T \log Z^U_\beta(\theta).
\end{align}
The argument of \autoref{thm:grad_expectation} applies verbatim with $\mathcal{S}$ replaced by $U$, so $\nabla_\theta A^U(\theta, \beta) = \E_{\rho^U_\beta}[\nabla_\theta E]$, where $\rho^U_\beta$ is the Gibbs distribution restricted to $U$ and renormalized.

\begin{proposition}\label{prop:local-limit}
Suppose $F(\theta, \beta, \cdot)$ attains its minimum over $\overline{U}$ at a unique interior point $s^\star_\beta(\theta)$, with a positive boundary gap: $F(\theta, \beta, s) > F(\theta, \beta, s^\star_\beta(\theta))$ for every $s \in \partial U$. Then, as $T \to 0$,
\begin{align}
A^U(\theta, \beta) \to F(\theta, \beta, s^\star_\beta(\theta)), \qquad \nabla_\theta A^U(\theta, \beta) \to \nabla_\theta E(\theta, s^\star_\beta(\theta)).
\end{align}
\end{proposition}

\begin{proof}
Apply Propositions \ref{prop:free-energy-limit} and \ref{prop:grad-limit} to the restriction of $F$ to $\overline{U}$. Uniqueness of the interior minimizer makes the minimizer set the single point $\{s^\star_\beta(\theta)\}$, so the convex hull in \autoref{prop:grad-limit} collapses to a point automatically.
\end{proof}

Subtracting the cases $\beta = 1$ and $\beta = 0$ of \autoref{prop:local-limit} gives the deterministic two-phase update around whichever local equilibrium is sampled by the network dynamics.

\subsection{Recovery of Infinitesimal EP}

\begin{proposition}\label{prop:small-beta}
Let $s_\beta(\theta)$ be a smooth branch of stationary points of $F(\theta, \beta, \cdot)$: $\nabla_s F(\theta, \beta, s_\beta(\theta)) = 0$ for $\beta$ in some open interval containing $0$, with smooth dependence of $s_\beta$ on $(\theta, \beta)$. (When the Hessian $\nabla^2_s F$ is nondegenerate along the branch, smooth dependence follows from the implicit function theorem.) Then
\begin{align}
\lim_{\beta \to 0} \frac{\nabla_\theta E(\theta, s_\beta(\theta)) - \nabla_\theta E(\theta, s_0(\theta))}{\beta} = \nabla_\theta \ell(s_0(\theta)).
\end{align}
\end{proposition}

\begin{proof}
Define 
\begin{align}
V_\beta(\theta) = F(\theta, \beta, s_\beta(\theta)) = E(\theta, s_\beta(\theta)) + \beta \ell(s_\beta(\theta)). 
\end{align}
The chain rule gives
\begin{align}
\nabla_\theta V_\beta = \nabla_\theta E(\theta, s_\beta) + \bigl[\nabla_s E(\theta, s_\beta) + \beta \nabla_s \ell(s_\beta)\bigr] \nabla_\theta s_\beta = \nabla_\theta E(\theta, s_\beta),
\end{align}
because the bracketed term is $\nabla_s F(\theta, \beta, s_\beta) = 0$ by stationarity. The same reasoning applied to the $\beta$-derivative gives $\partial_\beta V_\beta = \ell(s_\beta)$. Integrating,
\begin{align}
V_\beta(\theta) - V_0(\theta) = \int_0^\beta \ell(s_\alpha(\theta)) \dd\alpha,
\end{align}
and differentiating in $\theta$ under the integral,
\begin{align}
\frac{\nabla_\theta V_\beta - \nabla_\theta V_0}{\beta} = \frac{1}{\beta} \int_0^\beta \nabla_\theta \ell(s_\alpha(\theta)) \dd\alpha \xrightarrow{\beta \to 0} \nabla_\theta \ell(s_0(\theta))
\end{align}
by continuity. Substituting $\nabla_\theta V_\beta = \nabla_\theta E(\theta, s_\beta)$ gives the claim.
\end{proof}

Chaining \autoref{prop:local-limit} and \autoref{prop:small-beta} gives the iterated limit
\begin{align}
\lim_{\beta \to 0} \lim_{T \to 0} \nabla_\theta \left[\frac{A^U(\theta, \beta) - A^U(\theta, 0)}{\beta}\right] = \nabla_\theta \ell(s_0(\theta)),
\end{align}
which is exactly the classical infinitesimal EP identity of \citet{scellier2017equilibrium}. Our finite-temperature theory therefore contains classical EP as one corner of a two-parameter family indexed by $T$ and $\beta$.

\subsection{Degenerate Isolated Minima}

Finally, we treat the case where $F(\theta, \beta, \cdot)$ attains its global minimum $m$ at several isolated states. The finite-temperature Gibbs distribution still concentrates on this minimizer set, but the limiting weights on the individual minimizers are determined by local curvature.

\begin{proposition}\label{prop:degenerate}
Suppose $\mathcal{S} \subseteq \R^d$ and $F(\theta, \beta, \cdot)$ attains its minimum $m$ at finitely many isolated interior points $s_1, \dots, s_k$, with positive-definite Hessians $H_i = \nabla^2_s F(\theta, \beta, s_i)$. Then, for every continuous $h : \mathcal{S} \to \R$,
\begin{align}
\lim_{T \to 0} \E_{\rho_\beta}[h(s)] = \sum_{i=1}^k w_i h(s_i), \qquad w_i = \frac{\det(H_i)^{-1/2}}{\sum_{j=1}^k \det(H_j)^{-1/2}}.
\end{align}
Taking $h = \nabla_\theta E(\theta, \cdot)$,
\begin{align}
\lim_{T \to 0} \nabla_\theta A(\theta, \beta) = \sum_{i=1}^k w_i \nabla_\theta E(\theta, s_i).
\end{align}
\end{proposition}

\begin{proof}
A standard Laplace expansion. Near each $s_i$,
\begin{align}
F(s) = m + \tfrac{1}{2}(s - s_i)^\top H_i (s - s_i) + O(|s - s_i|^3),
\end{align}
and the Gaussian integral 
\begin{align}
\int \exp\bigl(-\tfrac{1}{2T} y^\top H_i y\bigr) \dd y = (2\pi T)^{d/2} \det(H_i)^{-1/2} 
\end{align}
gives, for $r$ small enough to isolate $s_i$,
\begin{align}
\int_{B(s_i, r)} \exp(-F/T) \dd s \sim (2\pi T)^{d/2} \det(H_i)^{-1/2} \exp(-m/T) \quad \text{as } T \to 0.
\end{align}
Contributions from outside the union of these balls are exponentially smaller because $F$ is bounded strictly above $m$ there. Summing over $i$ gives the leading behavior of $Z_\beta$, and the same expansion with an additional factor $h(s_i)$ (using continuity of $h$ to replace $h(s)$ by $h(s_i)$ inside the ball) gives the leading behavior of $\int h \exp(-F/T) \dd s$. Taking the ratio and dividing off the common factor $(2\pi T)^{d/2} \exp(-m/T)$ produces the stated formula.
\end{proof}

Two observations may be made. First, the limiting weights favor \emph{flatter} minima (smaller $\det H_i$); uniform sampling of degenerate minima is not the generic zero-temperature behavior and is recovered only in the special case where all Hessian determinants agree. Second, this curvature-weighted mixture provides a well-defined continuation of the zero-temperature limit through points where $m_\beta$ fails to be differentiable in $\theta$: the finite-temperature theory has a limit even when the deterministic minimum-value function does not have a classical gradient.

\subsection{Numerical Verification of Degenerate Minima}\label{app:degenerate-numerical}

To illustrate the zero-temperature behavior derived in \autoref{prop:degenerate}, we numerically integrated a continuous one-dimensional system. We defined an energy landscape with two isolated minima of exactly equal depth but unequal curvature. The state energy is $U(x) = (x^2 - 1)^2 \exp(k x)$ with an asymmetry parameter $k = 1.4$. This creates a flatter global minimum at $x = -1$ and a sharper global minimum at $x = +1$. The contrastive loss is defined as $\ell(x) = \tfrac{1}{2}(x - t)^2$ with a target $t = 1.0$.

To avoid any stochastic sampling noise and provide an exact baseline, we computed the true Gibbs density via trapezoidal quadrature. We evaluated the unnormalized log-density on a uniform grid of $24{,}001$ points spanning $x \in [-2.25, 1.85]$. We swept the temperature $T$ down from $0.8$ to $0.006$ on a logarithmic scale.

As shown in \autoref{fig:degenerate-minima}, the probability mass does not approach a uniform 50/50 split as $T \to 0$. Instead, the mass perfectly matches the theoretical Laplace weights: the flatter minimum captures approximately 80 percent of the probability mass, and the sharper minimum captures the remaining 20 percent. Because the limiting distribution is well-defined, the contrastive gradient also converges smoothly to a stable value.

\begin{figure}[htbp]
  \centering
  \includegraphics[width=\textwidth]{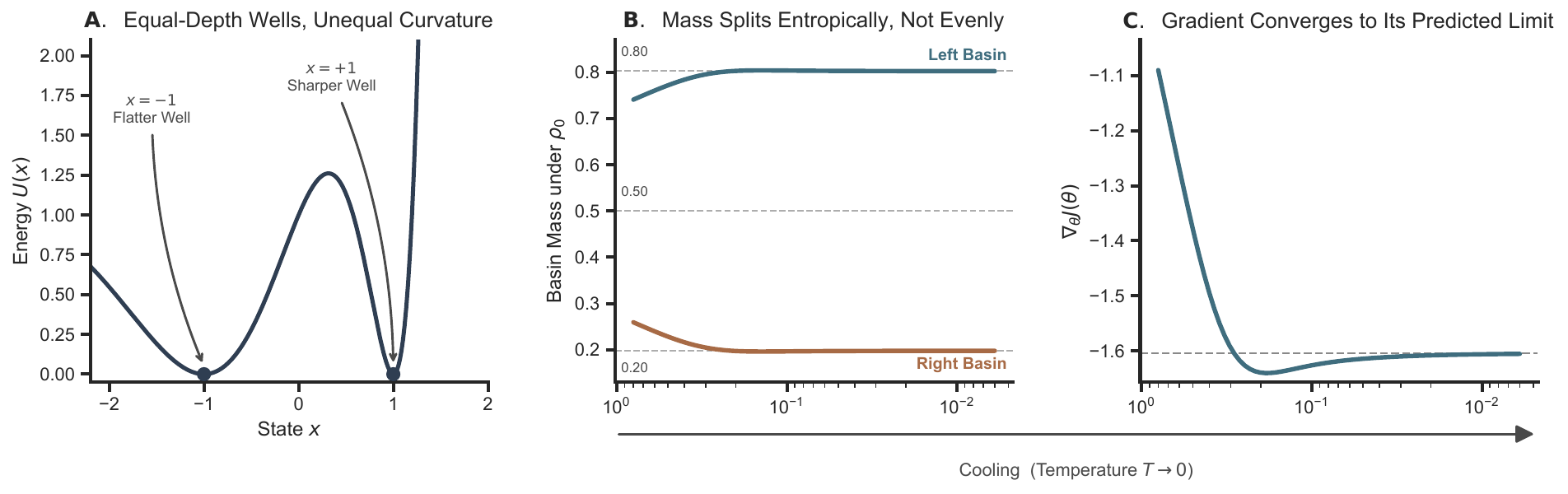}
  \caption{\textbf{Zero-temperature limit of a continuous system with degenerate minima.}
\textbf{(A)} The energy landscape $U(x) = (x^2 - 1)^2 \exp(1.4\, x)$ features two minima with identical depth ($U = 0$) but different local curvatures.
\textbf{(B)} The exact free-phase Gibbs probability mass contained in each basin as temperature $T$ decreases. The system does not sample the degenerate minima uniformly (grey dotted line at $0.5$). It instead converges to the curvature-weighted Laplace limits derived in \autoref{prop:degenerate}, favoring the flatter basin.
\textbf{(C)} The exact finite-nudge contrastive gradient $\nabla_\theta J(\theta)$ evaluated at $\beta = 1.0$. Because the probability mass converges smoothly to the Laplace limits, the thermodynamic gradient remains stable and well-defined as $T \to 0$.}
  \label{fig:degenerate-minima}
\end{figure}

\section{Exact Finite-Nudge Learning in a Marginalized Energy-Based Model}\label{app:exact-ebm}

To verify the theoretical results without the confounding effects of Markov Chain Monte Carlo sampling noise, we trained an energy-based classifier on the standard scikit-learn Digits dataset. The inputs are $8 \times 8$ grayscale pixel grids scaled to zero mean and unit variance.

The network is a single-hidden-layer energy-based model with $10$ binary hidden units. Because the hidden layer only has $1024$ possible states, we can sum over all hidden configurations exactly. This allows us to compute the exact partition function, the exact free energy, and the exact analytical gradients without relying on Langevin sampling.

The model was trained for $20$ epochs on $1000$ examples using the exact finite-nudge contrastive gradient, with a learning rate of $0.08$, weight decay of $10^{-4}$, batch size of $64$, nudging strength $\beta = 1.0$, and temperature $T = 1.0$.

To evaluate the signal-to-noise ratio (SNR) of standard sampling techniques, we compared the exact analytical gradient against Monte Carlo estimates generated using ancestral sampling from the exact hidden posterior. The SNR was calculated as the $L^2$ norm of the exact gradient divided by the root-mean-square error of the sampled gradient over $24$ independent trials.

To verify the exactness of the gradient identity at finite nudging strengths, we computed the directional derivative of the free-energy contrast using a central finite difference with $\varepsilon = 10^{-4}$. We compared this numerical derivative to the analytical expectation contrast derived in \autoref{sec:3}.

\begin{figure}[htbp]
  \centering
  \includegraphics[width=\textwidth]{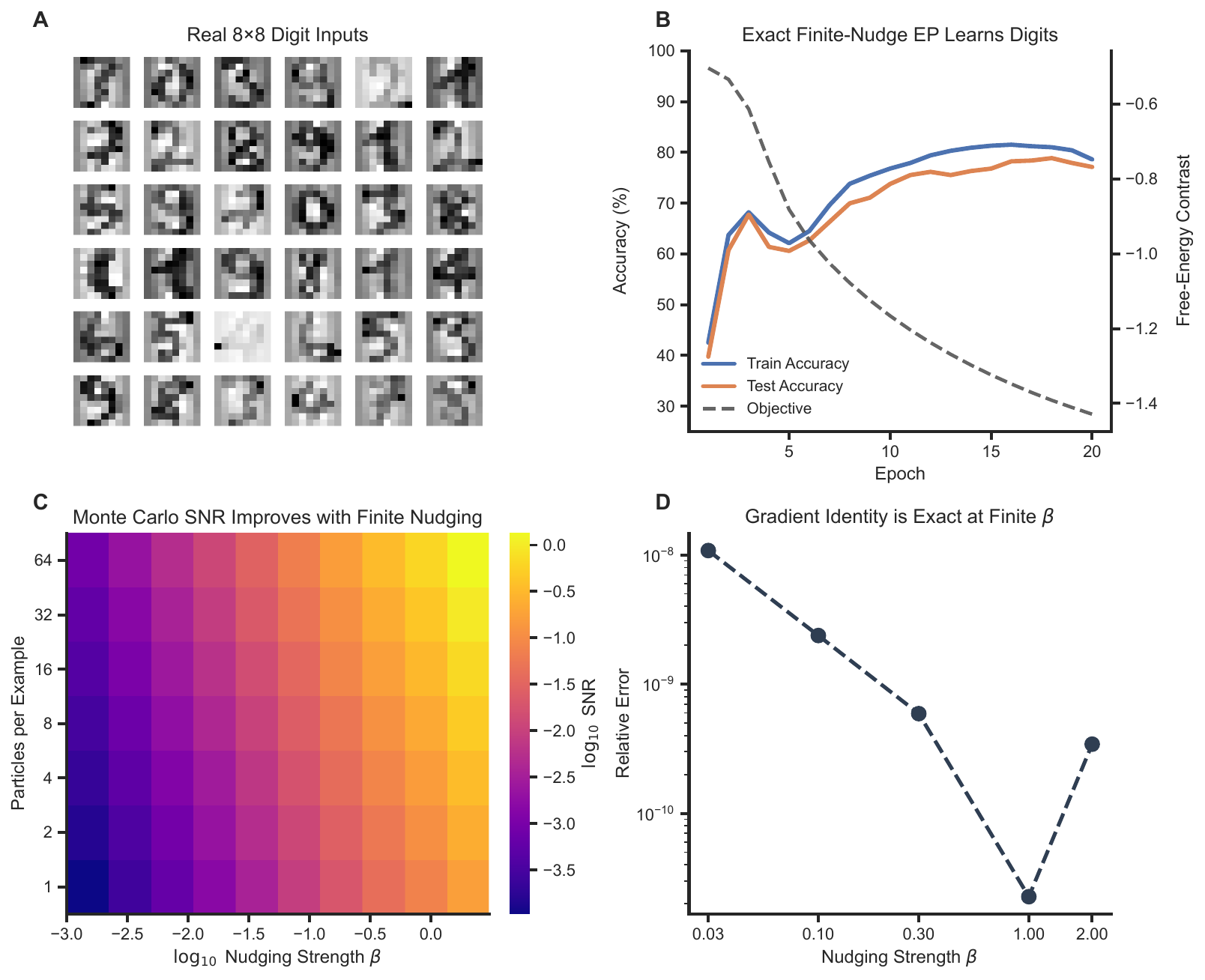}
  \caption{\textbf{Exact finite-nudge Equilibrium Propagation on the Digits dataset.}
\textbf{(A)} Example $8 \times 8$ digit inputs.
\textbf{(B)} Training accuracy, test accuracy, and the free-energy contrast objective evaluated over $20$ epochs.
\textbf{(C)} Monte Carlo signal-to-noise ratio of the gradient estimate across different nudging strengths and particle counts. Stronger nudging directly increases SNR.
\textbf{(D)} Relative error between the analytical gradient expectation and the numerical finite difference. The error remains at the limit of machine precision across all tested nudging strengths, confirming the gradient identity is mathematically exact.}
  \label{fig:exact-ebm}
\end{figure}

\section{Experimental Details for Fashion-MNIST}\label{app:fashion-mnist}

This appendix provides the complete methodological details and hyperparameters for the Fashion-MNIST experiments presented in \autoref{sec:results} and \autoref{fig:thermodynamic-validation}.

\paragraph{Dataset and Architecture.}
Fashion-MNIST consists of $28 \times 28$ grayscale images of clothing items. The inputs $x$ were flattened into $784$-dimensional vectors and scaled to the range $[-1, 1]$. We utilized a single-hidden-layer energy-based network with $128$ continuous hidden units $h \in \mathbb{R}^{128}$ and $10$ continuous output units $o \in \mathbb{R}^{10}$.

\paragraph{Energy and Loss Functions.}
The state energy $E$ is defined as a continuous Hopfield-style scalar energy:
\begin{align}
E(x, h, o) = \frac{1}{2}\|h\|^2 + \frac{1}{2}\|o\|^2 - h^\top(W_{\text{in}} x + b_h) - o^\top(W_{\text{out}} \tanh(h) + b_o).
\end{align}
The supervisory signal is provided by a mean-squared error loss against the one-hot encoded target $y$:
\begin{align}
\ell(o, y) = \frac{1}{2}\|o - \text{one\_hot}(y)\|^2.
\end{align}

\paragraph{Discrete Langevin Dynamics.}
In both the free and nudged phases, the network states are sampled from the Gibbs density $\rho_\beta(s) \propto \exp\!\bigl(-(E + \beta \ell)/T\bigr)$. We approximate this using discrete overdamped Langevin dynamics. A single Langevin step updates the joint state $s = (h, o)$ according to:
\begin{align}
s_{t+1} = s_t - \eta_s \nabla_s(E + \beta \ell) + \sqrt{2 T \eta_s}\, \xi, \qquad \xi \sim \mathcal{N}(0, I),
\end{align}
where $\eta_s = 0.08$ is the step size and $T = 0.03$ is the temperature. The free phase ($\beta = 0$) is run for $20$ steps. The nudged phase initializes its state from the final free-phase state and runs for $12$ steps.

\paragraph{Learning Rules and Baselines.}
For the contrastive updates, parameters were optimized using stochastic gradient descent with a learning rate of $0.05$. \autoref{fig:thermodynamic-validation} compares four estimators. Infinitesimal EP is evaluated at $\beta = 0.01$, with the gradient estimated by the scaled finite difference
\begin{align}
\hat{g}(\beta) = \frac{1}{\beta}\bigl(\E_{\rho_\beta}[\nabla_\theta E] - \E_{\rho_0}[\nabla_\theta E]\bigr).
\end{align}
Finite-nudge EP uses this same estimator at $\beta = 1.0$. The discrete path-integral variant evaluates the integral representation derived in \autoref{thm:grad_covariance}, approximating the continuous integral
\begin{align}
\grad_\theta J(\theta) = -\frac{1}{T}\int_0^\beta \Cov_{\rho_\alpha}\!\bigl[\ell, \nabla_\theta E\bigr] \dd\alpha
\end{align}
by trapezoidal quadrature over $5$ evenly spaced grid points from $\alpha = 0$ to $\alpha = 1.0$. Standard backpropagation runs a feedforward pass
\begin{align}
o = W_{\text{out}} \tanh(W_{\text{in}} x + b_h) + b_o
\end{align}
through the same architecture, minimizing cross-entropy loss with Adam at a learning rate of $0.002$.

\paragraph{Signal-to-Noise Ratio (SNR).}
For Panel B of \autoref{fig:thermodynamic-validation}, the SNR of the state perturbation $\Delta s = s_\beta - s_0$ was measured by running $12$ independent Langevin chains for each input. The SNR is defined as the $L^2$ norm of the mean perturbation across chains, divided by the root-mean-square of the noise (the standard deviation across the chains).

\end{document}